%% file: main.tex
\documentclass[review,onefignum,onetabnum]{siamart171218}
\usepackage{multirow}%
\usepackage{booktabs}
\newtheorem{assumption}{Assumption}

\input{ex_shared}

\ifpdf
\hypersetup{
  pdftitle={Harmonic Approximation of Discontinuous External Signals},
  pdfauthor={H. Jo, K. Josi\'{c}, and J. K. Kim}
}
\fi

\myexternaldocument{ex_supplement}

\begin{document}

\maketitle

\begin{abstract}
Non-autonomous differential equations are crucial for modeling systems influenced by external signals, yet fitting these models to data becomes particularly challenging when the signals change abruptly. To address this problem, we propose a novel parameter estimation method utilizing functional approximations with artificial neural networks. Our approach, termed Harmonic Approximation of Discontinuous External Signals using Neural Networks (HADES-NN), operates in two iterated stages. In the first stage, the algorithm employs a neural network to approximate the discontinuous signal with a smooth function. In the second stage, it uses this smooth approximate signal to estimate model parameters. HADES-NN gives highly accurate and precise parameter estimates across various applications, including circadian clock systems regulated by external light inputs measured via wearable devices and the mating response of yeast to external pheromone signals. HADES-NN greatly extends the range of model systems that can be fit to real-world measurements.
\end{abstract}



\section{Introduction}
Non-autonomous differential equations are used widely to model the evolution of natural and engineered systems subject to varying inputs and perturbations. Such models can provide a mechanistic understanding of medical and biological systems and predict their behavior (e.g., vaccine strategies in epidemic dynamics~\cite{du2017evolution,zinsstag2017vaccination} and cancer treatment approaches~\cite{farrokhian2022measuring}). Often, systems in nature are subject to abrupt changes in their inputs, such as the light intensity driving circadian pacemakers~\cite{forger1999simpler,lim2025enhanced}, pheromone signals in the mating response of yeast~\cite{pomeroy2021predictive}, and on-off switches in electrical circuits~\cite{zhao2022parameter}. However, these abrupt changes make non-autonomous models difficult to apply in practice.

In particular, fitting models with discontinuous inputs to data presents unresolved mathematical and computational challenges. Specifically, the presence of discontinuous inputs results in non-smooth loss functions, causing the failure of popular gradient descent-based optimization algorithms ~\cite{wang2022differentially} (e.g., Sequential Least Squares Quadratic Programming (SLSQP~\cite{kraft1988software}), Limited-memory Broyden–Fletcher–Goldfarb–Shanno algorithm with bound support (L-BFGS~\cite{zhu1997algorithm}), Levenberg-Marquardt method (LM~\cite{more2006levenberg}), and Neural Ordinary Differential Equations (NeuralODE~\cite{chen2018neural})). While direct search methods (e.g., Nelder-Mead (NM~\cite{gao2012implementing}) and differential evolution (DE~\cite{price2006differential})) do not use the gradient of the loss function, they often fail to converge in finite time because loss functions are non-smooth~\cite{garmanjani2013smoothing}. Artina \textit{et al.} proposed iterative smoothing and adaptive parametrization techniques to address non-smooth optimization problems~\cite{artina2013linearly}. However, this method is tailored for problems with linear constraints and is not directly applicable to optimization problems constrained by dynamic systems.

Recently, artificial neural networks have been used to fit a model of an electric circuit subject to discontinuous forcing to data~\cite{zhao2022parameter}. However, this approach requires extensive observations whenever the input to the circuit changes. The application of physics-informed neural networks has also been proposed to infer the underlying forcing term~\cite{song2024admm}, but only when the governing models are explicitly provided. Thus, a general, computationally efficient approach to fitting models with discontinuous inputs remains elusive.

Here, we present a novel method, the Harmonic Approximation of Discontinuous External Signals using Neural Networks (HADES-NN), to address this challenge. HADES-NN employs neural networks to approximate discontinuous inputs with a sequence of smooth functions \cite{imaizumi2019deep,adcock2021gap}. These approximate smooth functions are then used to replace the discontinuous input in the non-autonomous system, and the resulting sequence of systems is fitted to the data. Hence, HADES-NN minimizes a sequence of smooth loss functions instead of a single, non-smooth loss function. We prove that the sequence of minima produced by HADES-NN converges to the minimum of the original, non-smooth loss function, yielding accurate parameter estimates even with data is sparse. We illustrate the accuracy of HADES-NN using various examples: a Lotka-Volterra model driven by discontinuous seasonal input, a circadian pacemaker model driven by a discontinuous light signal measured using wearable devices, and a gene regulatory network model in mating yeast subject to abruptly changing pheromone levels. Our results indicate that HADES-NN is a flexible and computationally efficient method for fitting models with discontinuous inputs, greatly extending the range of model systems for which parameters can be estimated. By providing a robust and efficient solution to a long-standing problem, HADES-NN paves the way for more precise and reliable modeling of complex, dynamic systems. This advance has the potential to impact various fields, including biology, engineering, and environmental science, where accurate modeling of systems with discontinuous inputs is crucial.

\section{Main results}
\subsection{Leading optimization algorithms fail to infer parameters of a simple system with a discontinuous input}

\begin{figure}[!t]
	\centering
	\includegraphics[width=\columnwidth]{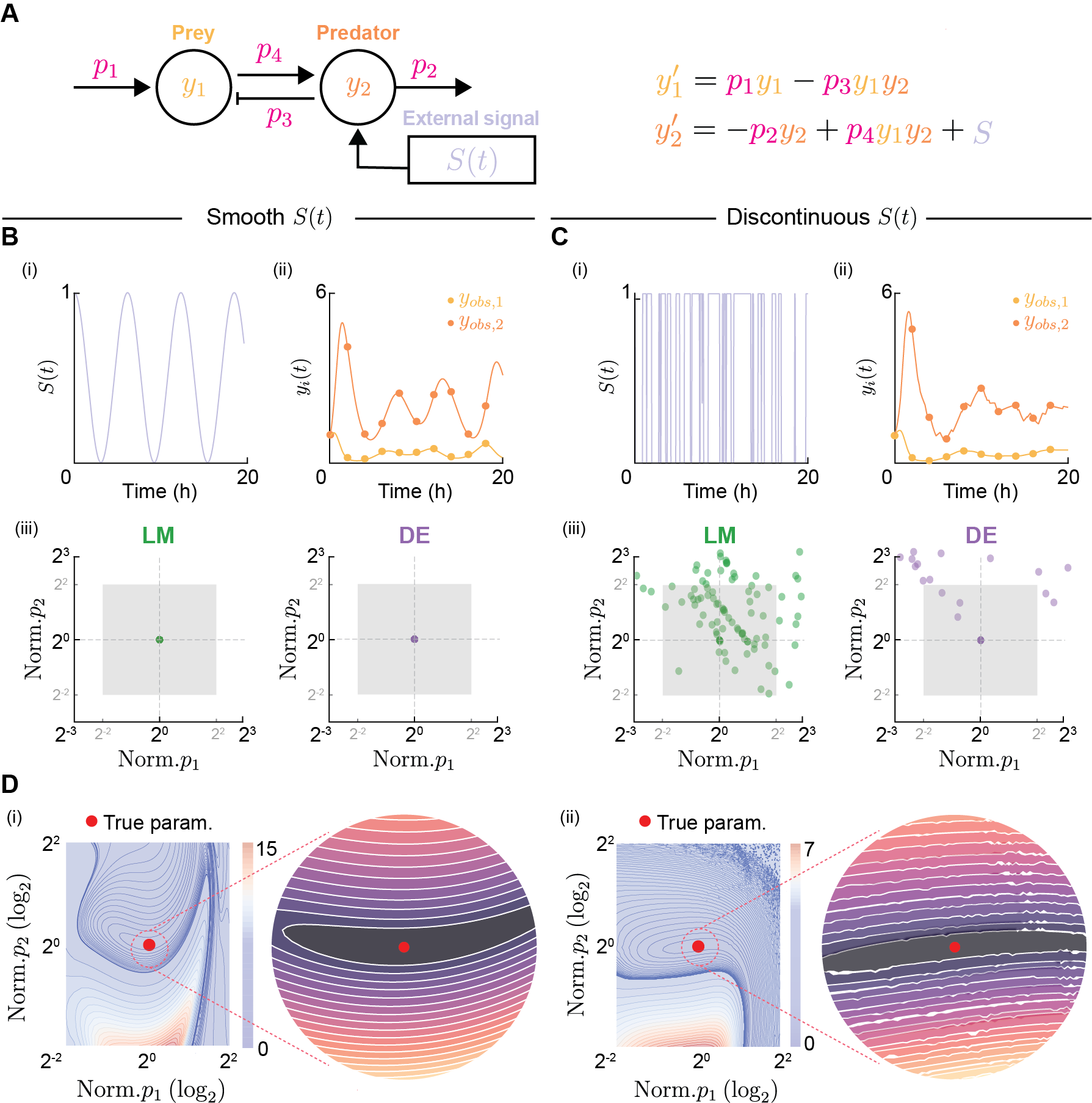}%
	\label{figure1}
	\caption{Parameter inference for the Lotka-Volterra system in a changing environment. (A) The Lotka-Volterra model with input $S(t)$ describes the interaction between prey ($y_1$) and predator ($y_2$) in a changing environment ($S(t)$). (B) When $S(t)$ is smooth (i), so is the corresponding numerical solution of the non-autonomous model (ii). In this case, the Levenberg-Marquardt (LM) and differential evolution (DE) algorithms successfully recover the parameters of the model from as few as 10 observations (iii). Here, the gray box represents the range of parameters from which initial values were randomly selected for optimization. Estimates were normalized by dividing with the true parameter values, located at the intersection of the dashed lines. The green and purple points indicate the final parameter estimates obtained using the LM and DE algorithms, respectively, each from independent random initializations (gray boxes). (C) In contrast, if $S(t)$ changes discontinuously (i), the resulting solution is not differentiable (ii). In this case, both LM and DE produce inaccurate parameter estimates (iii). The discontinuous signal $S(t)$ was generated using a continuous Markov process with a transition rate equal to 0.05 between the states $S(t)=0$ and $S(t)=1$ (See Fig. 5 for details). (D) The failure of these two standard optimization algorithms can be explained by the roughness of the loss landscape defined as the $L_2$ norm of the difference between the measurements and the numerical solution of the model system at different parameter values: With a smooth $S(t)$, the contours of the loss landscape are also smooth. However, with discontinuous input, $S(t),$ the contour lines are no longer smooth, resulting in numerous local minima.}
    \vspace{-2em}
\end{figure}

The Lotka-Volterra system describes the interaction between prey, $y_1 (t),$ and predator, $y_2 (t)$, species (Fig. 1A). Environmental impacts on the prey species, such as seasonal variations or human activity (e.g., hunting regulations), can be described by a signal, $S(t)$, system~\cite{sager2006numerical}. With a smooth signal, e.g., $S(t)=\frac{1}{2}\left(\cos{t}+1\right)$, the solution, $\Vec{y}(t)=(y_1(t),y_2(t))$ is also smooth (Fig. 1B, i-ii). Using ten observations from the simulated trajectories (dots in Fig. 1B, ii), we inferred the values of four parameters: $p_1$, the birth rate of prey, $p_2$, the death rate of predators, $p_3$, a measure of the effect of predators on prey growth, and $p_4,$ a measure of the effect of prey on predator growth (see Table I for the values of the parameters used to generate the data). Specifically, we estimated these four parameters with different initial guesses, randomly sampled from the volume $[\frac{p_i}{4},4p_i]$, $i=1,2,3,4$ (Fig. 1B, iii, gray regions), using a gradient-descent-based algorithm (LM) and a direct search algorithm (DE). Repeating this estimation procedure 100 times shows that both methods reliably produce accurate and precise parameter estimates near the true parameter values (Fig. 1B, iii crossed dashed lines).

However, both LM and DE fail when the input $S(t)$ is discontinuous. Specifically, when the signal $S(t)$ switches between $S=0$ and $S=1$ (Fig. 1C, i), the solution $\Vec{y}(t)=(y_1(t),y_2(t))$ becomes non-smooth (Fig. 1C, ii). Unlike the smooth input case, fitting the model using ten measurements (dots in Fig. 1C, ii) with LM and DE optimization algorithms yields inaccurate and imprecise estimates (Fig. 1C, iii).

This failure is typical of many classical optimization algorithms, which rely on some regularity of the loss function \cite{wu2017bolt, bassily2019private, feldman2020private, bassily2020stability, wang2022differentially}. To illustrate, we calculated the loss function: 
$$\left\Vert \Vec{y}_{obs}-\Vec{y}(\Vec{p})\right\Vert_2 = \left(\frac{1}{10}\sum_{j=1}^{10} \sum_{i=1}^2\left[y_{obs,i}(t_j^o)-y_i(t_j^o;\Vec{p})\right]^2\right)^{1/2},$$ 
where $\Vec{y}_{obs}(t_j^o)=(y_{obs,1}(t_j^o), y_{obs,2}(t_j^o))$ and $\Vec{y}(p)=(y_{1}(t_j^o;\Vec{p}), y_{2}(t_j^o;\Vec{p}))$ are the observations and the numerical solution of the system with the set of parameters $\Vec{p}=(p_1,p_2,p_3,p_4)$ at times $t_j^o \in [0,20]$, respectively. Although $\Vec{y}$ in fact depends on both $t$ and $S(t)$, we omit this explicit dependence for the sake of notational simplicity. For the complete expression and further details, please refer to the Methods section or Appendix. When $S(t)$ is smooth, then so is the loss landscape, and the parameters estimated with both optimization algorithms converge to the minimum of the loss function (Fig. 1D, i). However, when $S(t)$ is discontinuous, the loss function is non-smooth, and both optimization algorithms fail to identify the direction in the parameter space that minimizes the loss function and, hence, do not converge to its minimum (Fig. 1D, ii).

\subsection{HADES-NN: Neural network-based parameter estimation for non-autonomous differential equations with discontinuous signals} 

\begin{figure}
	\centering
	\includegraphics[width=\columnwidth]{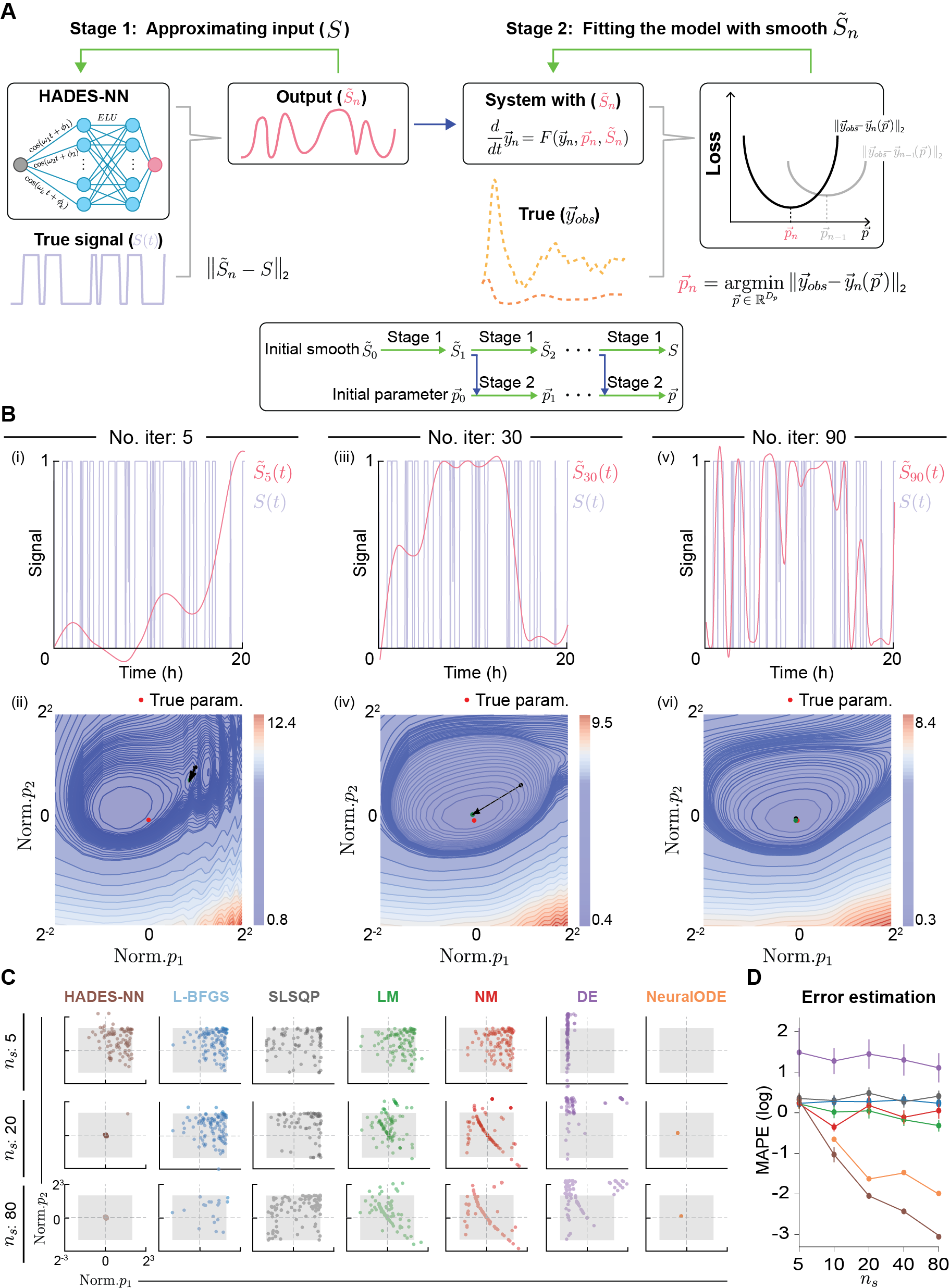}%
	\label{figure2}
        \vspace{-1em}
	\caption{Harmonic Approximation of Discontinuous External Signals using Neural Networks (HADES-NN). (A) HADES-NN iterates over two stages. The first stage generates a smooth approximation, $\tilde{S}_n(t),$ of the discontinuous external signal, $S(t),$ using a neural network (i and ii). The second stage uses LM to estimate the parameters, $\Vec{p}_{n},$ of the system $\frac{d}{dt}\Vec{y}_n(t) = F(\Vec{y}_n(t),\Vec{p}_n, \tilde{S}_n(t))$ and initial parameter values obtained in the previous iteration, $\Vec{p}_{n-1}$. Stage 1 and Stage 2 are executed sequentially in an iterative loop, where the output of each stage serves as the input for the next (green arrows). In particular, $\tilde{S}_n$ is used as the input to the system for estimating $\Vec{p}_n$ (blue arrow). (B) HADES-NN is then used to estimate the parameters of the Lotka-Volterra model, defined in Fig. 1C.  Initially, $\tilde{S}_n(t)$ is a poor approximation of the discontinuous signal $S(t)$ (i), leading to parameter estimates (ii, black arrow) far from the true values (ii, red dot). As $\tilde{S}_n$ converges to $S(t)$ over iterations, the parameter estimates converge to their true values (iv, v). (C) The scatter plots show final parameter estimates obtained using seven optimization algorithms, each run multiple times with different initial parameter values, for three different observation numbers (5, 20, and 80). Only NeuralODE and HADES-NN recover the true parameter values (intersection of dashed lines) as the number of observations increases. (D) The Mean Absolute Percentage Error (MAPE) between the true and the estimated parameters. HADES-NN produces estimates with the smallest error.}
    \vspace{-2em}

\end{figure}

A non-smooth loss landscape is known to cause the failure of many current optimization algorithms ~\cite{kraft1988software, zhu1997algorithm, more2006levenberg, chen2018neural, gao2012implementing, price2006differential}. To address this challenge, HADES-NN estimates parameters of differential equations with discontinuous inputs, $S(t),$ by first smoothing the inputs and the corresponding loss function. Specifically, HADES-NN alternates between two stages (Fig. 2A): Stage 1) approximating the discontinuous input $S(t)$ with a smooth function $\tilde{S}_n(t)$ using a neural network with cosine activations in the signal layer and Exponential Linear Unit (ELU) activations in five hidden layers (Fig. 2A, i-ii), where ELU is defined as $\text{ELU}(x) = x$ if $x > 0$, and $(e^x - 1)$ if $x \leq 0$ (See Section 3.2 for a detailed description), and Stage 2) fitting the model with the approximate smooth input $\tilde{S}_n(t)$ using LM, yielding the parameter estimate, $\Vec{p}_n$ (Fig. 2A). In the next iteration, we use the approximate input obtained in the previous step, $\tilde{S}_n(t)$, to initialize the neural network in the first signal approximation stage, and $\Vec{p}_n$ as the initial parameter guess in the second parameter estimation stage. By repeating these two stages, the neural network produces increasingly accurate, smooth approximations, $\tilde{S}_n(t),$ of the discontinuous input, $S(t)$. As the smooth approximations $\tilde{S}_n(t)$ converge to $S(t)$, the parameter estimates $\Vec{p}_n$ converge to the true minimum of the loss function (See Supplementary information for a proof of this assertion).

To illustrate its performance, we applied HADES-NN to the example Lotka-Volterra system described above (Fig. 1C). After five iterations, the approximate signal $\tilde{S}_5(t)$ obtained by the neural network was a poor approximation of the original discontinuous signal $S(t)$ (Fig. 2B, i). Hence, the resulting loss landscape differed significantly from the true loss landscape, leading to a poor approximation (Fig. 2B, ii-black arrow) of the true parameters (Fig. 2B, ii-red dot) in the LM stage of the algorithm. On the other hand, as the number of iterations increases to 30 and 90, HADES-NN produced more accurate approximations of the input $S(t)$ (Fig. 2B, iii, v), resulting in a better match between the approximate and true loss landscapes' global minimizers (Fig. 2B, iv, vi). As a result, the parameter estimates (Fig. 2B, iv, vi-black arrows) converge to their true values (Fig. 2B, iv, vi-red dots). Since smooth approximations of $S(t)$ were used at each stage of the algorithm, the approximate loss landscapes were also smooth, enabling LM to identify the minimum of the approximate, smooth loss function at each iteration step.

We compared the performance of HADES-NN with six popular optimization algorithms: Three gradient-based algorithms (L-BFGS, SLSQP, LM), two direct search-based algorithms (NM, DE), and one deep learning-based algorithm (NeuralODE). We used different numbers of observations ($n_{s} = 5, 20, 80$) of the simulated trajectories of the Lotka-Volterra system with the discontinuous input, $S(t)$ (Fig. 1A, right). Using these data, we estimated parameters with initial guesses randomly sampled 100 times from the volume $[\frac{p_i}{4},4p_i]$, $i = 1,2,3,4$ (Fig. 2C, gray regions).  

We compared the performance of HADES-NN with six popular optimization algorithms: Three gradient-based algorithms (L-BFGS, SLSQP, LM), two direct search-based algorithms (NM, DE) and one deep learning-based algorithm (NeuralODE). We used different numbers of observations ($n_{s}=5, 20, 80$) of the simulated trajectories of the Lotka-Volterra system with the discontinuous input, $S(t)$ (Fig. 1A, right). Using these data, we estimated parameters with initial guesses randomly sampled 100 times from the volume $[\frac{p_i}{4},4p_i]$, $i=1,2,3,4$ (Fig. 2C, gray regions).  

With $n_{s}=5$, all methods failed to produce accurate parameter estimates due to lack of data. With more observations ($n_{s}=20, 80$), the estimates produced with HADES-NN and NeuralODE, but not the other methods, converged to the true parameter values (intersection of the dashed lines in (Fig. 2C)). HADES-NN produced more accurate estimates compared to NeuralODE when we quantified the accuracy of estimates using the Mean Absolute Percentage Error (MAPE), representing the average relative difference between the true and estimated parameters (Fig. 2D). 

\subsection{HADES-NN accurately estimates parameters of a human circadian clock model with abruptly changing light exposure}

\begin{figure}[!t]
	\centering
	\includegraphics[width=\columnwidth]{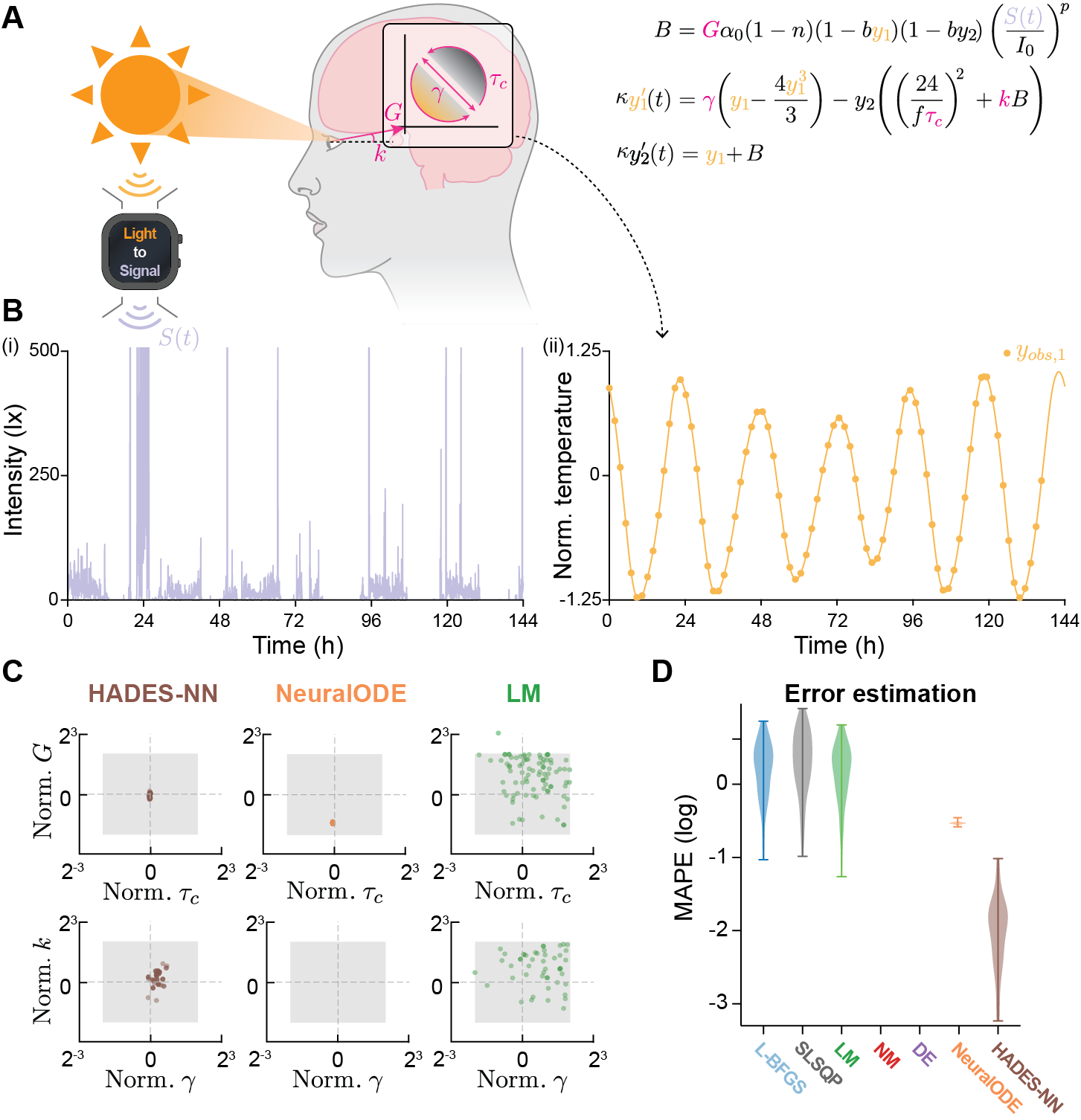}%
	\label{figure3}
	\caption{Parameter estimation for a circadian pacemaker model. (A) The circadian clock in humans is driven by light exposure $S(t)$. Variations among individuals are determined by four parameters: the circadian period $\tau_c$, oscillation stiffness $\gamma$, gain factor $G$, and relative effect $k$. (B) We obtained light exposure measurements from a wearable device (left) and generated 80 observations by solving the circadian pacemaker model with these light measurements as inputs (right). (C) The scatter plots show estimated parameters after 100 trials for each of the three methods (HADES-NN, NeuralODE, and LM), with randomly selected initial parameter values from the gray box. Only HADES-NN recovered the true values (intersection of dashed lines). (D) The MAPE between the true and estimated parameters shows that HADES-NN outperforms the other methods we tested. Direct search algorithms (NM and DE) failed to converge, so the scatter plots were not included.}
    \vspace{-2em}

\end{figure}

We next illustrate the performance of HADES-NN on a more challenging problem. To do so, we consider a human circadian pacemaker model \cite{forger1999simpler}, widely used to address human health challenges such as jet lag \cite{skeldon2017effects, christensen2020optimal}, and shift-worker sleep disorder \cite{cheng2021predicting, huang2021predicting, hong2021personalized, song2023real,song2025digital} (Fig. 3A). The internal circadian clock is perturbed by light exposure, $S(t)$, which increases the fraction of the activated photoreceptors, $n(t)$, through a driving signal modeled by the term $B(t)$. This signal modulates the activity of pacemaker neurons, $y_2(t)$, and the dynamics of the auxiliary variable, $y_1(t)$. While $y_2(t)$ is not observable, $y_1(t)$ is proportional to core body temperature \cite{forger1999simpler}. Therefore, we assume that the state of the system is not fully observable and only used measurements of $y_1(t)$ as data (Fig. 3B, ii).

The dynamics of the circadian clock are determined by four key parameters, which represent unique characteristics of each individual: the intrinsic circadian period ($\tau_c$); the oscillator stiffness ($\gamma$), determining the stability of the oscillator against external driving forces; the gain factor ($G$), determining how light impacts $B(t)$; and the relative effect ($k$), which denotes the ratio between the magnitudes of the influence of light on the modulations of $y_1(t)$ and $y_2(t)$. By estimating these four parameters for each individual, we can analyze the characteristics of their circadian rhythms (e.g., morning type, sensitivity to jet lag, etc.) \cite{mott2011model, bonarius2020parameter, stone2020role, skeldon2022extracting, skeldon2023method}. Previous studies used particle filters or the NM method to estimate model parameters but assumed that light exposure, $S(t)$, was unrealistically smooth \cite{mott2011model, bonarius2020parameter, skeldon2023method}.

Accurate parameter estimates using realistic inputs for individual patients are needed to enable personalized treatments of sleep disorders. To achieve this goal, we used a discontinuous profile of light exposure $S(t)$ obtained from a wearable device (Actiwatch2, Phillips Respironics, Murrysville, PA) (Fig. 3B, i). With this signal as input, we simulated the normalized core body temperature $y_1(t)$ (Fig. 3B, ii) with previously obtained parameter values, $\tau_c$, $\gamma$, $G$, and $k$ (Table II). We then sampled 80 points from the resulting trajectory as data. Using the simulated data, we estimated the four parameters with seven optimization algorithms (L-BFGS, SLSQP, LM, NM, DE, NeuralODE, and HADES-NN). We repeated this process 100 times with different initial parameter guesses, each sampled uniformly from the interval $[\frac{p_i}{4},4p_i]$, for each parameter $p_i\in\Vec{p}= (\tau_c, \gamma, G, k)$. 

Among the algorithms we tested, only HADES-NN consistently generated accurate and precise parameter estimates (Fig. 3C). In contrast, the estimates of the parameters $\gamma$ and $k$ using NeuralODE diverged outside the range $[2^{-3},2^3]$, indicating its failure to recover the true parameter values with a model more complex than the Lotka-Volterra system (Fig. 3C) (See supplementary information for estimation results of other algorithms, Fig. S1). HADES-NN also outperformed the other algorithms when we quantified the accuracy of estimates using MAPE (Fig. 3D). In particular, the direct search algorithms (NM and DE) failed to converge. This indicates that, among the algorithms we tested, only HADES-NN can be used to estimate the parameters of this circadian clock model with discontinuous input. 

\subsection{HADES-NN accurately estimates parameters of a model of mating response in yeast}
\begin{figure}[!h]
	\centering
	\includegraphics[width=\columnwidth]{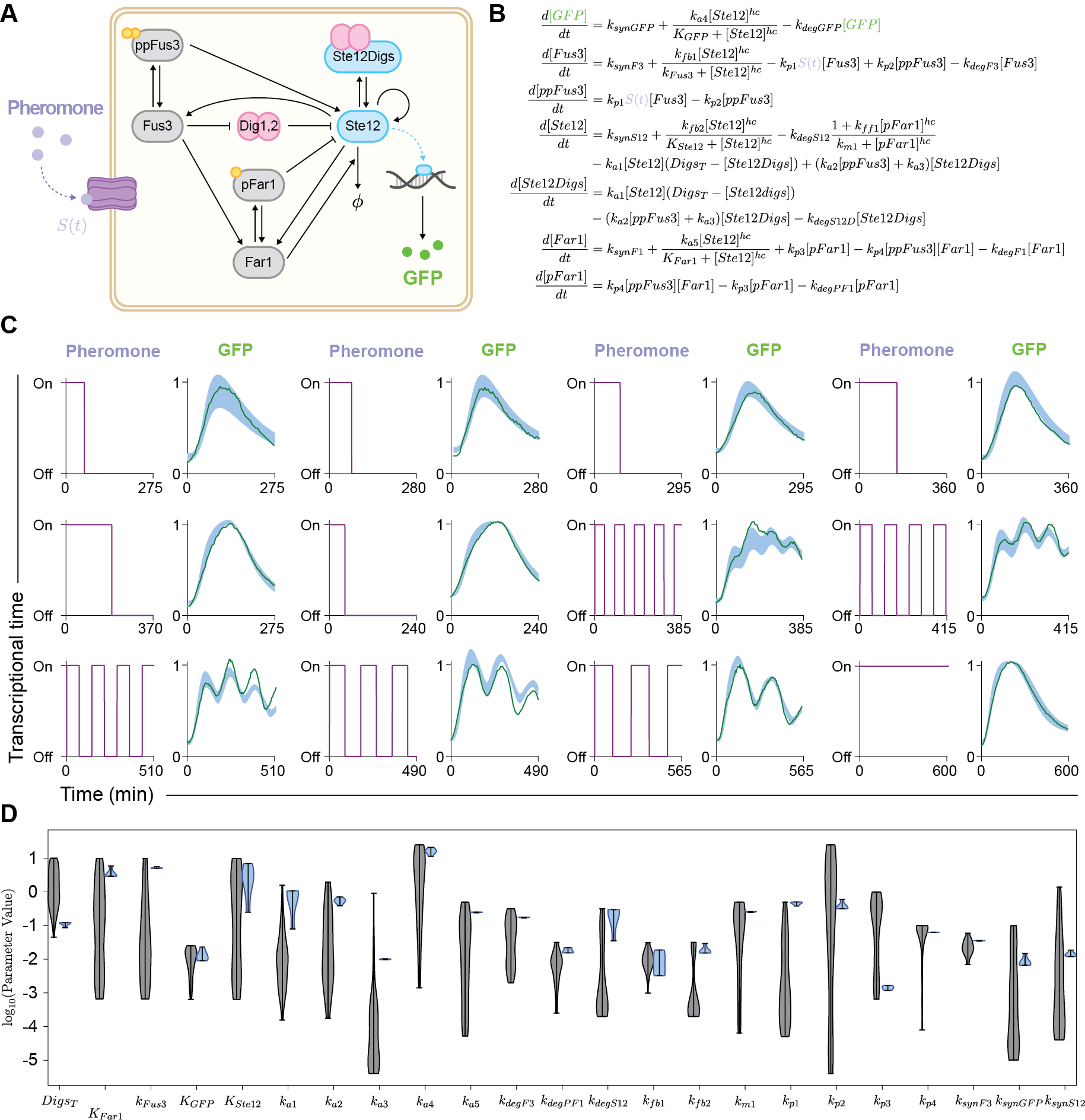}%
	\label{figure4}
	\caption{Parameter estimation in the mating response of yeast network. (A) When pheromone $S(t)$ binds to receptors on the cell surface (left, purple), it activates a signaling pathway, triggering a cascade of intracellular events. This cascade involves the phosphorylation and activation of various signaling proteins and transcription factors, such as the MAP kinase Fus3 and the transcription factor Ste12. The activation of these proteins leads to the activation of Green Fluorescent Protein (GFP) (right, green). (B) These interactions can be described by a system of differential equations \cite{pomeroy2021predictive}. (C) To estimate the parameters of the model shown in (B), we used experimental data describing how the pheromone dosage (input signal, purple) affects the amount of GFP (output signal, green). The shaded ranges represent the mean $\pm$ one standard deviation of trajectories simulated with parameters estimated by HADES-NN. (D) The gray violin plots represent the estimated ranges of parameters using an evolutionary algorithm \cite{pomeroy2021predictive}. In comparison, HADES-NN results in a more precise and significantly narrower range of parameter estimates (green violin plots).}
    \vspace{-2em}
\end{figure}

To demonstrate that HADES-NN performs well with even more complex models and experimentally measured data, we used it to estimate the parameters of a mating response network model of yeast~\cite{pomeroy2021predictive}. This model was introduced to understand the impact of pheromones on gene expression in \emph{Saccharomyces cerevisiae} (yeast) and incorporates network motifs such as incoherent feedforward loops, positive feedback loops, negative feedback loops, as well as the slow rebinding of transcriptional repressors (Fig. 4A). Specifically, pheromone with concentration denoted by $S(t)$ serves as a cell-to-cell communication signal, initiating preparation for mating in yeast cells. When pheromones bind to cell surface receptors, the signaling cascade begins, leading to the activation of the MAP kinase Fus3 (ppFus3), which then promotes several intracellular responses, propagating the pheromone $S(t)$ and activating a transcription factor (Ste12). Ste12 plays a critical role in regulating the expression of genes associated with mating and is repressed by forming a complex with Dig1 and Dig2 (Ste12Digs). In experiments, Green Fluorescent Protein (GFP) expressed in response to the degree of Ste12 activation is used as a reporter, allowing for visual confirmation of the activation of the signaling pathway. These interactions can be described by a system of seven differential equations with 23 unknown parameters (Fig. 4B). 

In a previous study, to estimate these large sets of parameters, a rich set of experimental data was generated by producing different patterns of input to the system $S(t)$ and measuring the resulting GFP response~\cite{pomeroy2021predictive} (Fig. 4C). These data were used along with initial parameter guesses selected randomly within given ranges (Table III) to estimate the model parameters. An evolutionary algorithm was run repeatedly with different initial guesses to produce 100 sets of parameter estimates~\cite{pomeroy2021predictive}. However, the estimated parameter ranges were large (Fig. 4D, gray violin plots). In comparison, HADES-NN produced more precise parameter estimates (Fig. 4D, green violin plots). Importantly, the simulated trajectories using the estimated parameters accurately captured experimental observations (Fig. 4C, blue region), supporting the accuracy of the parameters estimated with HADES-NN. The accuracy of previously estimated parameters could not be assessed as only the ranges of these parameters in Fig. 4D were provided~\cite{pomeroy2021predictive}. The accurate and precise estimates produced by HADES-NN indicate that this algorithm performs well even with complex, biophysically realistic models. This allows for better insights into the cellular responses to environmental changes, pinpointing abnormal signaling pathways associated with diseases and informing the development of targeted therapeutic strategies.
\section{Methodology}
\subsection{Description of the general non-autonomous differential equations with an external signal $S(t)$}
A system subject to an external signal, $S(t)$, can generally be described by a non-autonomous differential equation:
\begin{equation}\label{general_de}
	\frac{d}{dt}\Vec{y}(t) =F(\Vec{y}(t),\Vec{p},S(t)),
\end{equation}
where we denote the unknown set of parameters by $\Vec{p}=(p_1,...,p_{D_p})\in\mathbb{R}^{D_p}$. Our goal is to infer the parameter values, $\Vec{p}_{*},$ that minimize the distance between the solution of Eq.~(\ref{general_de}) and measurements of the system at discrete points in time. We used the $L_2$ norm ($\left\Vert \cdot \right\Vert_2$) of their difference as the loss function:
\begin{align}
	\Vec{p}_{*} & =\argmin_{\Vec{p}\in\mathbb{R}^{D_p}}\left\Vert \Vec{y}_{obs}(\cdot)-\Vec{y}(\cdot;\Vec{p},S(\cdot))\right\Vert_{L_2\left(\{ t_j^o\}_{j=1}^{N} \right)}\label{p_star_loss}\\
	& =\argmin_{\Vec{p}\in\mathbb{R}^{D_p}}\left()\frac{1}{N}\sum_{j=1}^{N} \sum_{i=1}^{D_y}\left[y_{obs,i}(t_j^o)-y_i\left(t_j^o;\Vec{p},S(t_{j}^{o})\right)\right]^2\right)^{1/2},\nonumber
\end{align}

where \( \{t_j^o\}_{j=1}^{N} \) represents the times of the observations (measurements), and \( \Vec{y}_{\text{obs}}(t_j^o) = (y_{\text{obs},1}(t_j^o), \dots, y_{\text{obs},D_y}(t_j^o)) \) and \( \Vec{y}(t_j^o; \Vec{p},S(t_j^o)) = (y_1(t_j^o; \Vec{p},S(t_j^o)), \dots, y_{D_y}(t_j^o; \Vec{p},S(t_j^o))) \) denote the observed data and the model solution of the system Eq.~(\ref{general_de}) at those times, respectively.
\subsection{Description of HADES-NN architecture}\label{hades-nn_method}

Previously proposed optimization algorithms fail to recover the true parameters because of the roughness of the loss landscape in the case of a discontinuous external signal. To resolve this issue, we constructed a sequence of \textit{smooth} signals, $\tilde{S}_n(t),$ that approximate the original discontinuous signal, $S(t)$ \cite{imaizumi2019deep}. HADES-NN employs a fully connected neural network (Fig. 2A, Stage 1) to produce a sequence of approximations, $\tilde{S}_n(t)$, of the discontinuous signal $S(t)$ at times $t^s_i$ when $S(t)$ is observed $(i=1,..., M)$. Specifically, the smooth signal $\tilde{S}_n(t)$ is constructed as follows: The input time variable $t_i^s$ is fed into the first hidden layer, consisting of 16 nodes. In this layer, each node uses a cosine activation function with varying periods $\omega_l$ and phase shifts $\phi_l$, for $l = 1,...,16$ (See also \cite{sitzmann2020implicit, tancik2020fourier}). The outputs of these 16 nodes serve as inputs to nodes in the second layer with 16 nodes. We chose ELU activation in the second layer, which is crucial in making the output of the network smooth. We subsequently construct four more hidden layers with the same structure and repeat the same process as in the second layer. Finally, the last hidden layer is mapped to a single node approximating the input, $S(t)$. Thus, the architecture of HADES-NN comprises one signal layer of size one, five hidden layers, each containing 16 nodes, and one output layer of size one. Since we train HADES-NN with Eq.~(\ref{L_theta}) via backpropagation, $\tilde{S}_n(t)$ converges to $S(t)$ as $n$ increases, with the difference measured by the mean-square distance,
\begin{equation}\label{L_theta}
	\left\Vert S(\cdot)-\tilde{S}_{n}(\cdot)\right\Vert_{L_2\left(\{t_i^s\}_{i=1}^{M}\right)}:=\frac{1}{M}\sum_{i=1}^{M} \left|S(t^s_i) - \tilde{S}_n(t^s_i) \right|^2.
\end{equation}

We next replace the external signal, $S(t),$ in the non-autonomous differential equation with the smooth approximation, $\tilde{S}_n(t),$ and fit the resulting model to data (Fig. 2A, Stage 2). Hence, in the $n^{th}$ iteration, we use the model
\begin{equation}\label{general_de_smooth}
	\frac{d}{dt}\Vec{y}_n(t) =F(\Vec{y}_n(t),\Vec{p}_n, \tilde{S}_n(t)).
\end{equation}
We then use the Levenberg-Marquardt algorithm (LM) to find the global minimizer, $\Vec{p}_n=(p_{n,1},\dots,p_{n,D_p}),$ of the loss function measuring the $L_2$ difference between the observed values $\Vec{y}_{obs}(t),$ and the numerical solutions of Eq.~(\ref{general_de_smooth}), $\Vec{y}_n(t;\Vec{p}_n,\tilde{S}(t))=(y_{n, 1}(t;\Vec{p}_n,\tilde{S}(t)),\dots,y_{n, D_y}(t;\Vec{p}_n,\tilde{S}(t)))$ (Fig. 2A, Stage 2):
\begin{align}
	\Vec{p}_n & =\argmin_{\Vec{p}\in\mathbb{R}^{D_p}}\left\Vert \Vec{y}_{obs}(\cdot)-\Vec{y}_n\left(\cdot;\Vec{p}, \tilde{S}_n(\cdot)\right)\right\Vert_{L_2\left(\{ t_j^o\}_{j=1}^{N} \right)}\label{p_n_loss}\\
	& =\argmin_{\Vec{p}\in\mathbb{R}^{D_p}}\left(\frac{1}{N}\sum_{j=1}^{N} \sum_{i=1}^{D_y}\left[y_{obs,i}(t_j^o)-y_{n,i}\left(t_j^o;\Vec{p}, \tilde{S}_n(t_{j}^{o})\right)\right]^2\right)^{1/2},\nonumber
\end{align}
where $t^o_j$ represents the time when $\Vec{y}_{obs}(t)$ is measured. To find the minimizer of Eq.~(\ref{p_n_loss}), we used randomly initialized parameters, $\Vec{p}_{0},$ for the first iteration ($n=1$). In subsequent iterations ($n>1$), we used the estimated parameters, $\Vec{p}_{n},$ obtained at the end of each iteration as the initial parameters for LM at the next iteration (Fig. 2A, Stage 2, right). We repeat this process until the difference between two consecutive estimates is sufficiently small, 

\begin{equation}\label{tolerance}
	\left(\sum_{k=1}^{D_p}\left(p_{n,k}-p_{n-1,k}\right)^2\right)^{1/2}\leq \varepsilon,
\end{equation}

where $0<\varepsilon\ll 1$ denotes a tolerance. We show that the estimates $\Vec{p}_n$ converge to the true parameter values as $\tilde{S}_n$ converges to $S(t)$ (See Supplementary information for mathematical analysis).

\section{Discussion}
Our method is based on approximating a discontinuous function with a sequence of smooth functions with increasing total variation (Fig. 2A). While this can be accomplished using several techniques, including splines, deep learning approaches offer a comparable level of effectiveness \cite{sahs2022shallow}. Notably, unlike other methods, deep learning approaches do not require adjustment between applications or even different data sets within a single application. Furthermore, by collecting a single fixed estimated parameter value at each stage (Fig. 2B, i-iii, red dots), we can extend HADES-NN to present the estimation results as a range of possible true parameter values, allowing for uncertainty quantification in estimated parameters.

In addition to discontinuous external signals, rapidly changing dynamical systems can also create non-smooth loss landscapes. For example, Krishnapriyan {\it et al.} showed that the diffusion equation with high diffusion coefficients results in non-smooth loss landscapes when neural networks are used to approximate the solution of the system \cite{raissi2019physics, krishnapriyan2021characterizing, karniadakis2021physics}. To address this issue, Krishnapriyan {\it et al.} proposed to gradually increase the diffusion coefficient from a low initial value to its final target value during training. When the diffusion coefficient is small, the loss landscape is smooth, making it easier to train the neural networks. The trained neural networks at this stage are then used as initial networks for the next stage with a larger diffusion coefficient. By iterating this process, the sequence of neural networks converges to the solution of the diffusion equation with high diffusion coefficients. While this approach focuses on the diffusion coefficient (i.e., a single fixed value), HADES-NN extends its applicability to discontinuous signals (i.e., a function over time). 

While HADES-NN utilizes numerical integration for estimating parameters in dynamical systems, deep learning-based parameter estimation methods, such as Physics-Informed Neural Networks (PINNs), have recently gained attention due to their effectiveness in handling various forms of parameter estimation (e.g., \cite{raissi2019physics, jo2019deep, hwang2020trend}). Specifically, PINNs have been successfully applied not only to constant parameter estimation but also to time-varying parameters and even to distributional parameter inference \cite{jung2020real, jo2024density}. Integrating the HADES-NN approach into an end-to-end PINN framework is expected to extend its applicability to a wider range of parameter types, potentially enhancing model generalizability across diverse dynamical systems.

\section{Conclusion}
We have shown that HADES-NN outperforms current state-of-the-art algorithms for parameter estimation in models with discontinuous inputs. Abrupt changes in the input result in non-smooth loss landscapes, leading to the failure of popular optimization algorithms with various approaches (SLSQP, LBFGS-B, LM, NM, DE, and NeuralODE). On the other hand, HADES-NN accurately estimates parameters for systems with discontinuous inputs. Its successful performance with complex models suggests its potential utility in various scientific and engineering domains.

\appendix
\section{Notation}
Let $\Vec{p}=\left( p_1, p_2, \dots, p_{D_p}\right)\in\mathbb{R}^{D_p}$ be a real-valued vector with dimension $D_{p}\geq 1$. Here, we use the $L_2$ norm (magnitude) of $\Vec{p}$ defined as,
$$\Vert\Vec{p}\Vert_2=\left(\sum_{i=1}^{D_p}p_i^{2}\right)^{1/2}.$$
For the time-dependent vector $\Vec{y}(t)=\left( y_1(t), y_2(t), ... y_{D_y}(t)\right)$ on a time interval $t\in[0, T]$, we use the $L_2$ norm of $\Vec{y}(t)$ over the time interval $[a,b]\subset[0,T]$, $\Vert\Vec{y}(\cdot)\Vert_{L^2([a,b])}$, defined as:

$$\Vert\Vec{y}(t)\Vert_2=\left(\sum_{i=1}^{D_y}y_i^2(t)\right)^{1/2},$$
$$\Vert\Vec{y}(\cdot)\Vert_{L_2([a,b])}=\left(\int_{a}^{b}\Vert\Vec{y}(v)\Vert_2^2 dv\right)^{1/2}.$$

When $\Vec{y}_{obs}(t)=\left(y_{obs,1}(t), y_{obs,2}(t), ... y_{obs,D_y}(t)\right)$ is observed at discrete time points $t^{o}_{1},\cdots,t^{o}_{N}\in[0,T]$, we use the $L_2$ norm of the vector of observed values:
$$\Vert\Vec{y}_{obs}(\cdot)\Vert_{L_2\left(\{t^{o}_{j}\}_{j=1}^{N}\right)}=\left(\frac{1}{N}\sum_{j=1}^{N} \sum_{i=1}^{D_y} y_{obs,i}^{2}(t^{o}_{j}) \right)^{1/2}.$$

\section{The parameters inferred at each stage of the outer loop of the HADES-NN converge to the true parameters} 
This section provides a mathematical proof of the claim that the sequence of parameters $\Vec{p}_n$ in Eq.~(\ref{p_n_loss}), obtained at the $n^{th}$ iteration of HADES-NN using a smooth approximation of the input, $\tilde{S}_n(t)$, converges to the minimum of the actual loss function, $\Vec{p}_{*},$ in Eq.~(\ref{p_star_loss}) under reasonable assumptions.
The proof consists of three steps showing that:

Step 1. The loss function $\left\Vert \Vec{y}_{obs}(\cdot)-\Vec{y}_n(\cdot,\Vec{p})\right\Vert_{L_2\left(\{t^{o}_{j}\}_{j=1}^{N}\right)}$ defined in Eq.~(\ref{p_n_loss}) is smooth in the parameter $\Vec{p}$ when the signal $\tilde{S}_n(t)$ is smooth. Thus at each step of the iteration, the loss function is smooth. 

Step 2. With smooth approximate input, $\tilde{S}_n(t),$ in the $n^{th}$ iteration, the optimization problem is solvable using the LM algorithm, i.e., gradient-based parameter updates can be used to reach the optimal parameter, $\Vec{p}_n=\argmin_{\Vec{p}\in\mathbb{R}^{D_p}} \left\Vert \Vec{y}_{obs}-\Vec{y}_n(\Vec{p})\right\Vert_2$.

Step 3. The sequence of minimizers, $\Vec{p}_n$, obtained in Step 2 of the algorithm converges to the minimizer of the original loss function, $\Vec{p}_{*}$.

Step 1 is an immediate result from the following proposition \ref{smoothness}. Specifically,
\begin{proposition}\label{smoothness} Let $F$ in Eq.~(\ref{general_de}) be  continuously differentiable with respect to $i$-th component $y_i$ of $\Vec{y}$ and $j$-th component $p_j$ of $\Vec{p}$, for all $i=1,\cdots,d_y$ and $j=1,\cdots,d_p$, and $\Vec{y}(0)=\Vec{y}_{0}$ does not dependnt on $\Vec{p}$. Then the partial derivative of $y_i$ with respoect to $p_j$,  $z_{ij} = \frac{\partial y_i}{\partial p_j}$, satisfy the following differential equation:
\begin{align}
    \frac{\partial z_{ij}}{\partial t} & = \frac{\partial F_i}{\partial p_j} + \sum_{k=1}^{d_y}\frac{\partial F_i}{\partial y_k}z_{kj},\nonumber\\
    z_{ij}(0)&=0. \nonumber
\end{align}
\end{proposition}
\begin{proof}
     See Eqs. (1–13) in \cite{dickinson1976sensitivity} (or Theorem 6.3.2 in \cite{akhmet2010principles} for a more general formulation.
\end{proof}
Since $\tilde{S}_{n}(t)$ is a smooth function, the resulting vector field $F$ is also smooth. Hence, by Proposition~\ref{smoothness}, the partial derivative $\frac{\partial y_i}{\partial p_j}$ exists for all $t > 0$, which establishes the validity of Step 1.

Step 2 also follows from the following Proposition \ref{global_LM} (Theorem 3.1 in Bergou et al. \cite{bergou2020convergence}). 
\begin{proposition}\label{global_LM} Given tolerance  $\varepsilon\in (0,1)$, the index $n_\varepsilon$ of the first iteration satisfying the following stopping criteria:
$$\left\|\nabla_{\Vec{p}} \left\Vert \Vec{y}_{obs}(\cdot)-\Vec{y}_n(\cdot,\Vec{p}_{n_\varepsilon})\right\Vert_{L_2\left(\{t^{o}_{j}\}_{j=1}^{N}\right)}\right\|_2\leq\varepsilon$$
is bounded by:
$$n_\varepsilon \leq C\varepsilon^{-2}$$
where $C$ is s a problem-dependent constant determined by the initial value, the model reduction acceptance threshold, and the sufficient decrease constant.
\end{proposition}

For the Step 3, we provide a direct proof under the following assumptions:

\begin{assumption}\label{assumption1} (Identifiability of the parameter) Given the signal $S(t)$, the minimizer $\Vec{p}_{*}$ of Eq.~(\ref{p_star_loss}) is unique.
\end{assumption}
\begin{assumption}\label{assumption2} (Regularity of $F$ in Eq.~(\ref{general_de})) The function $F$ defined in Eq.~(\ref{general_de}) is Lipschitz continuous with respect to $\Vec{y}(t)$ and $S(t)$. That is, there exists two constants $C_{1}$ and $C_{2}$ such that
$$\left\Vert F(\Vec{y}_{1}(t),\Vec{p},S(t))-F(\Vec{y}_{2}(t),\Vec{p},S(t))\right\Vert_2\leq C_{1} \left\Vert \Vec{y}_{1}(t)-\Vec{y}_{2}(t)\right\Vert_2, $$
$$\left\Vert F(\Vec{y}(t),\Vec{p},S_{1}(t))-F(\Vec{y}(t),\Vec{p},S_{2}(t))\right\Vert_2 \leq C_{2} |S_{1}(t)-S_{2}(t)|, $$
for all fixed time points $t\in[0,T]$.
\end{assumption}
Under these two assumptions, we have

\begin{lemma}\label{difference_wrt_S} Let $\Vec{p}\in\mathbb{R}^{D_p}$ be a fixed vector (parameter), and $y_1$ and $y_2$ be two solutions of Eq.~(\ref{general_de}) with $(\Vec{p},S_1(t))$ and $(\Vec{p},S_2(t))$, respectively. Assume that the two solutions have the same initial condition $\Vec{y}_{1}(0) = \Vec{y}_{2}(0)=\Vec{y}_{o}(0)$. That is,
\begin{align}
    \frac{d}{dt}\Vec{y}_{1}(t) &= F(\Vec{y}_{1}(t),\Vec{p},S_{1}(t)), \label{de_y1}\tag{7}\\
    \frac{d}{dt}\Vec{y}_{2}(t) &= F(\Vec{y}_{2}(t),\Vec{p},S_{2}(t)). \label{de_y2}\tag{8}
\end{align}
If the two input signals, $S_1(t)$ and $S_2(t)$, are Riemann integrable, then:
\begin{equation}\label{ineq_diff_y}\tag{9}
    \left\Vert \Vec{y}_{1}-\Vec{y}_{2}\right\Vert_{{L_2([0,T])}} \leq C(T) \left\Vert S_{1}-S_{2}\right\Vert_{{L_2([0,T])}},
\end{equation}
where $C(T)$ is a constant depending only on $T$ and the Lipschitz constants of $F$.
\end{lemma}

Notably, the right hand side in ~\cref{difference_wrt_S} is independent of the vector $\Vec{p}$. 

\begin{proof}
    Substracting Eq.~(\ref{de_y1}) from Eq.~(\ref{de_y2}) and integrating over $[0, t]$, yields the following relationship
\begin{align}
    \int_{0}^{t}\frac{d}{dv}\left(y_1(v)-y_2(v)\right)dv &= \int_{0}^{t}\left(F(\Vec{y}_{1}(v),\Vec{p},S_{1}(v)) - F(\Vec{y}_{1}(v),\Vec{p},S_{2}(v)) \right)dv \nonumber \\
    &+ \int_{0}^{t}\left(F(\Vec{y}_{1}(v),\Vec{p},S_{2}(v)) - F(\Vec{y}_{2}(v),\Vec{p},S_{2}(v))\right)dv\nonumber
\end{align}
Since $\int_{0}^{t}\frac{d}{dv}\left(y_1(v)-y_2(v)\right)dv=\left(\Vec{y}_{1}(t)-\Vec{y}_{2}(t)\right) - \left(\Vec{y}_{1}(0)-\Vec{y}_{2}(0)\right)$ and $\Vec{y}_{1}(0)=\Vec{y}_{2}(0)$ (by assumption), we have
\begin{align}
    \Vec{y}_{1}(t)-\Vec{y}_{2}(t) &= \int_{0}^{t}\left(F(\Vec{y}_{1}(v),\Vec{p},S_{1}(v)) - F(\Vec{y}_{1}(v),\Vec{p},S_{2}(v)) \right)dv \nonumber \\
    &+ \int_{0}^{t}\left(F(\Vec{y}_{1}(v),\Vec{p},S_{2}(v)) - F(\Vec{y}_{2}(v),\Vec{p},S_{2}(v))\right)dv.\nonumber
\end{align}
Next, we take the absolute value, square both sides, and use the inequality $(a+b)^{2}\leq 2 (a^2+b^2)$ on the right hand to obtain,
    \begin{align}
    & \left\Vert \Vec{y}_{1}(t)-\Vec{y}_{2}(t)\right\Vert_2 ^2 \label{ineq_F}\tag{10} \\
    &\leq 2 \left\Vert \int_{0}^{t}\left[F(\Vec{y}_{1}(v),\Vec{p},S_{1}(v)) - F(\Vec{y}_{1}(v),\Vec{p},S_{2}(v)) \right]dv\right\Vert_2^2 \nonumber \\
    &+ 2 \left\Vert\int_{0}^{t}\left[F(\Vec{y}_{1}(v),\Vec{p},S_{2}(v)) - F(\Vec{y}_{2}(v),\Vec{p},S_{2}(v))\right]dv\right\Vert_2^2. \nonumber
\end{align}
For the first term in the right hand side in Eq.~(\ref{ineq_F}), recall that the definition of $\|\cdot\|_2$,
\begin{align}
    & \left\Vert \int_{0}^{t}\left[F(\Vec{y}_{1}(v),\Vec{p},S_{1}(v)) - F(\Vec{y}_{1}(v),\Vec{p},S_{2}(v)) \right]dv\right\Vert_2^2\nonumber\\
    & = \sum_{i=1}^{D_y}\left(\int_{0}^{t}\left[F_{i}(\Vec{y}_{1}(v),\Vec{p},S_{1}(v)) - F_{i}(\Vec{y}_{1}(v),\Vec{p},S_{2}(v)) \right] dv\right)^2.\tag{11}\label{ineq_holder}
\end{align}
By H\"{o}lder's inequality $\left( \int_{0}^{t} |f(v)\cdot g(v)| dv \leq \left(\int_{0}^{1}f^2(v)dv\right)^{1/2}\cdot\left(\int_{0}^{1}g^2(v)dv\right)^{1/2}\right)$ inside each square term, we have
\begin{align}
    & \int_{0}^{t}1\cdot\left[F_{i}(\Vec{y}_{1}(v),\Vec{p},S_{1}(v)) - F_{i}(\Vec{y}_{1}(v),\Vec{p},S_{2}(v))\right]dv\nonumber\\
    & \leq \left(\int_{0}^{t} 1 dv\right)^{1/2}\cdot\left(\int_{0}^{t}\left[F_{i}(\Vec{y}_{1}(v),\Vec{p},S_{1}(v)) - F_{i}(\Vec{y}_{1}(v),\Vec{p},S_{2}(v))\right]^2dv\right)^{1/2}\nonumber\\
    &\leq T^{1/2}\left(\int_{0}^{t}\left[F_{i}(\Vec{y}_{1}(v),\Vec{p},S_{1}(v)) - F_{i}(\Vec{y}_{1}(v),\Vec{p},S_{2}(v))\right]^2dv\right)^{1/2}.\nonumber
\end{align}
By applying the above inequality into Eq.~(\ref{ineq_holder}),
\begin{align}
    & \left\Vert \int_{0}^{t}\left[F(\Vec{y}_{1}(v),\Vec{p},S_{1}(v)) - F(\Vec{y}_{1}(v),\Vec{p},S_{2}(v)) \right]dv\right\Vert_2^2\nonumber\\
    & \leq \sum_{i=1}^{D_y}T\int_{0}^{t}\left[F_{i}(\Vec{y}_{1}(v),\Vec{p},S_{1}(v)) - F_{i}(\Vec{y}_{1}(v),\Vec{p},S_{2}(v))\right]^2dv\nonumber\\
    &= T \int_{0}^{t}\sum_{i=1}^{D_y}\left[F_{i}(\Vec{y}_{1}(v),\Vec{p},S_{1}(v)) - F_{i}(\Vec{y}_{1}(v),\Vec{p},S_{2}(v))\right]^2dv\nonumber\\
    &= T \int_{0}^{t}\left\Vert F(\Vec{y}_{1}(v),\Vec{p},S_{1}(v)) - F(\Vec{y}_{1}(v),\Vec{p},S_{2}(v)) \right\Vert_2 ^2 dv \nonumber\\
    & \leq C_2 T \int_{0}^{t}\left|S_{1}(v)-S_{2}(v)\right|^2 dv \label{ineq_S}\tag{12}
\end{align}
Here, the last inequality in Eq.~(\ref{ineq_S}) can be derived from Assumption ~\ref{assumption2} (Lipschitz continuity of $F$ with respect to $S$).

Similarly, we obtain the following inequalities
\begin{align}
   &  \left\Vert \int_{0}^{t}\left[F(\Vec{y}_{1}(v),\Vec{p},S_{2}(v)) - F(\Vec{y}_{2}(v),\Vec{p},S_{2}(v)) \right]dv\right\Vert_2 ^2 \nonumber \\
   &  \leq C_2 T \int_{0}^{t}\left\Vert \Vec{y}_{1}(v)-\Vec{y}_{2}(v)\right\Vert_2 ^2 dv \label{ineq_y}\tag{14}
\end{align}

By combining the three inequalities ~(\ref{ineq_F})-(\ref{ineq_y}), we obtain:
\begin{align}
& \left\Vert \Vec{y}_{1}(t)-\Vec{y}_{2}(t)\right\Vert_2 ^2 \leq \nonumber \\
& C(T) \left(\int_{0}^{t}\left|S_{1}(v)-S_{2}(v)\right|^2 dv+\int_{0}^{t}\left\Vert \Vec{y}_{1}(v)-\Vec{y}_{2}(v)\right\Vert_2 ^2 dv \right), \label{ineq_simple}\tag{13}
\end{align}
where $C(T)$ is a constant depending only on $T$ and Lipschitz constants, $C_1$ and $C_2$, of $F$. 
By applying Gr\"{o}nwall's Inequality (see also Theorem A in \cite{ye2007generalized}) to Eq.~(\ref{ineq_simple}) over time interval $[0,T]$, we obtain an upper bound on $\left\Vert \Vec{y}_{1}-\Vec{y}_{2}\right\Vert_{{L_2([0,T])}}$, leading to the desired inequality given by Eq.~(\ref{ineq_diff_y}).
\end{proof}

\begin{theorem}\label{convergence}
    Let $\{ \Vec{p}_n \}_{n=1}^{\infty}$ be the sequence of global minimizers of Eq.~(\ref{p_n_loss}), obtained at the $n^{th}$ iteration of HADES-NN using a smooth approximation of the input, $\tilde{S}_n(t)$. Then, $\Vec{p}_n$ converges to the global minimizer $\Vec{p}_{*}$ of Eq.~(\ref{p_star_loss}).
\end{theorem}

\begin{proof}
    Given $\varepsilon>0$, we assume we can find an integer $N>0$ so that $\left\Vert \tilde{S}_{n}(\cdot)-S(\cdot)\right\Vert_{L_2\left(\{t^{o}_{j}\}_{j=1}^{N}\right)}<\varepsilon$, for all $n>N$. This result was proved in \cite{park2020minimum}, and is a universal approximation theorem for neural networks. Let $\Vec{y}(t;\Vec{p}_n,S(t))$ be the solution of Eq.~(\ref{general_de}) with $\Vec{p}_n$ and $S(t)$. Then, 
\begin{align}
    \left\Vert \Vec{y}_{obs}(\cdot) - \Vec{y}(\cdot;\Vec{p}_n,\tilde{S}_{n}(\cdot)) \right\Vert_{L_2\left(\{t^{o}_{j}\}_{j=1}^{N}\right)} 
    & \leq \left\Vert \Vec{y}_{obs}(\cdot) - \Vec{y}(\cdot;\Vec{p}_{*},\tilde{S}_{n}(\cdot)) \right\Vert_{L_2\left(\{t^{o}_{j}\}_{j=1}^{N}\right)} \nonumber \\
    & \leq \left\Vert \Vec{y}_{obs}(\cdot) - \Vec{y}(\cdot;\Vec{p}_{*},S(\cdot)) \right\Vert_{L_2\left(\{t^{o}_{j}\}_{j=1}^{N}\right)} + \left\Vert \Vec{y}(\cdot;\Vec{p}_{*},S(\cdot)) - \Vec{y}(\cdot;\Vec{p}_{*},\tilde{S}_{n}(\cdot)) \right\Vert_{L_2\left(\{t^{o}_{j}\}_{j=1}^{N}\right)} \nonumber \\
    & \leq \left\Vert \Vec{y}_{obs}(\cdot) - \Vec{y}(\cdot;\Vec{p}_{*},S(\cdot)) \right\Vert_{L_2\left(\{t^{o}_{j}\}_{j=1}^{N}\right)} + C\varepsilon. \nonumber 
\end{align}

The first inequality follows from the definition of the global optimum of the loss function defined in Eq.~(\ref{p_n_loss}) for $(\Vec{p}_n,S_{n}(t))$. The second inequality follows from the triangle inequality. The third inequality follows from ~\cref{difference_wrt_S}.
 
Thus, we can obtain the following inequality:

\begin{equation}\nonumber
   \left\Vert \Vec{y}_{obs}(\cdot) - \Vec{y}(\cdot;\Vec{p}_n,\tilde{S}_{n}(\cdot)) \right\Vert_{L_2\left(\{t^{o}_{j}\}_{j=1}^{N}\right)} \leq \left\Vert \Vec{y}_{obs}(\cdot) - \Vec{y}(\cdot;\Vec{p}_{*},S(\cdot)) \right\Vert_{L_2\left(\{t^{o}_{j}\}_{j=1}^{N}\right)} + C\varepsilon
\end{equation}

By taking $\lim_{n\rightarrow\infty}$ of both sides, and $\varepsilon \rightarrow 0$, we have

\begin{equation}\nonumber
    \lim_{n\rightarrow\infty}\left\Vert \Vec{y}_{obs}(\cdot) - \Vec{y}(\cdot;\Vec{p}_n,S_{n}(\cdot)) \right\Vert_{L_2\left(\{t^{o}_{j}\}_{j=1}^{N}\right)} \leq \left\Vert \Vec{y}_{obs}(\cdot) - \Vec{y}(\cdot;\Vec{p}_{*},S(\cdot)) \right\Vert_{L_2\left(\{t^{o}_{j}\}_{j=1}^{N}\right)}
\end{equation}

This implies the sequence of $\{ \Vec{p}_n \}$ converges to a specific vector, then the vector is also the optimal solution of Eq.~(\ref{p_n_loss}) since $\Vec{p}$ is the optimal solution (by Assumption ~\ref{assumption1}). This proves ~\cref{convergence}.

\end{proof}

\section{Supplementary figures and tables}

\begin{figure}[h]
	\centering
	\includegraphics[width=\columnwidth]{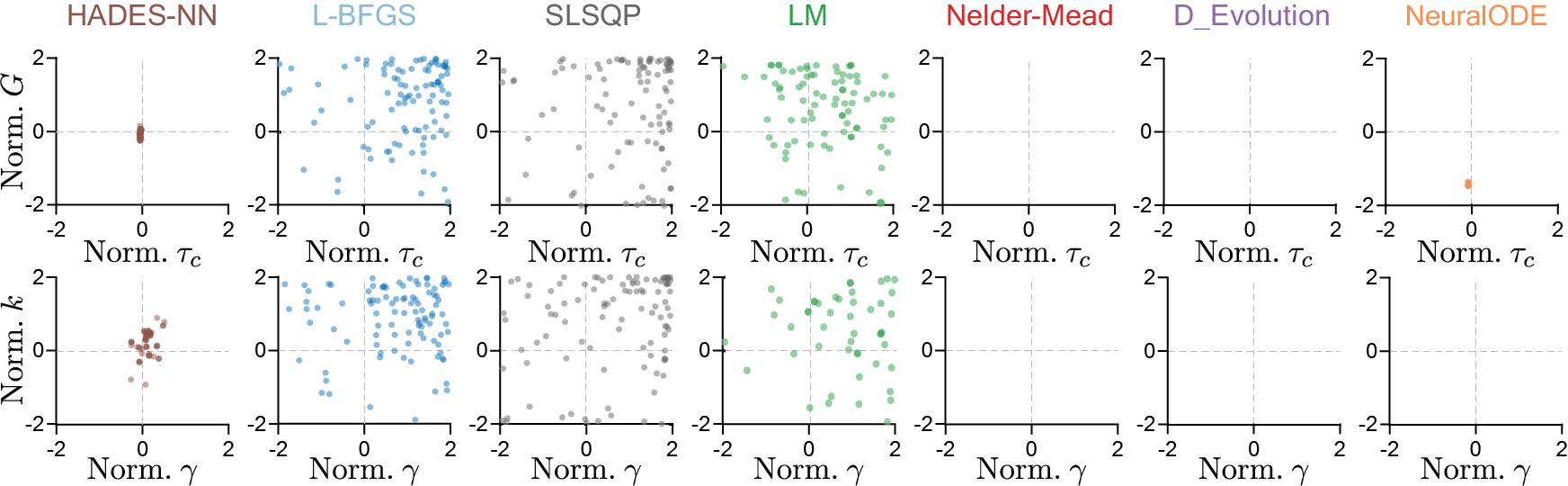}%
	\label{figureS1}
	\caption{Parameter estimation in circadian pacemaker model. The scatter plots of the estimated parameters after performing 100 trials with different initial parameter values for each of the seven methods (HADES-NN, L-BFGS, SLSQP, Levenberg-Marquardt (LM), Nelder-Mead (NM), Differential Evolution (DE), and NeuralODE). The parameter estimation results using direct search algorithms (NM and DE) failed to converge.}
    \vspace{-2em}
\end{figure}

\begin{table}[h!]\centering
\caption{Parameter values of the Lotka-Volterra model used to generate observation data points (yellow dots in Fig. 1B and C, ii). We generate $S(t)$ in Fig. 1C using the transition matrix $T$. Starting with the initial state $S(t_0)=1$, the next state $S(t_1)$ is determined based on the transition probabilities in $T$. Specifically, the system remains in the same state with a probability of 0.95 and transitions to the other state with a probability of 0.05.}
\begin{center}
\begin{small}
\begin{sc}
\begin{tabular}{cc}
Parameter & Values\\
\midrule
$p_1$ & 2\\
$p_2$ & 1/2\\ 
$p_3$ & 1\\
$p_4$ & 1\\
\midrule
\multicolumn{2}{c}{Transition matrix for $S(t)$}\\
\multicolumn{2}{c}{
$
   T = \begin{bmatrix} 
   0.95 & 0.05  \\
   0.05 & 0.95  \\
   \end{bmatrix}
$
}\\

\bottomrule
\end{tabular}
\end{sc}
\end{small}
\end{center}
\end{table}

\begin{table}\centering
\caption{Parameter values of the circadian pacemaker model used to generate observations (yellow dots in Fig. 3B, ii). Those values are adopted from \cite{forger1999simpler, skeldon2017effects}. The first four parameters were estimated while the others were assumed to be known. }
\begin{center}
\begin{small}
\begin{sc}
\begin{tabular}{ccc}

Parameter & Setting & Estimated \\
\midrule
$\tau_c$ & 20 & Y\\
$\gamma$ & 0.23 & Y\\ 
$G$ & 20 & Y\\
$k$ & 0.55 & Y\\
$\alpha_0$ & 0.16 & N\\ 
$b$ & 0.4 & N\\ 
$I_0$ & 9500 & N\\ 
$p$ & 0.6 & N\\ 
$\kappa$ & $\frac{12}{\pi}$ & N\\
$f$ & 0.99669 & N\\
\bottomrule
\end{tabular}
\end{sc}
\end{small}
\end{center}
\end{table}

\begin{table}\centering
\caption{Parameter ranges for the model of the yeast mating response network. The ranges of parameters were selected according to the previous study \cite{pomeroy2021predictive}. Initial values of parameters were randomly selected from these ranges for the estimation.}
\begin{center}
\begin{small}
\begin{sc}
\begin{tabular}{cc}
Parameter & Ranged allowed ($\log_{10}$)\\
\midrule
$K_{synF3}$ & $-6.9 \sim -1.6$ \\
$k_{fb1}$ & $-10.3 \sim -1.7$ \\
$K_{Fus3}$ & $-4.5 \sim 0.0$ \\
$k_{p1}$ & $-5.0 \sim -1.0$ \\
$k_{p2}$ & $-5.0 \sim -1.0$ \\
$k_{degF3}$ & $-3.6 \sim -1.5$ \\
$k_{synS12}$ & $-7.1 \sim -1.6$ \\
$k_{fb2}$ & $-10.5 \sim -1.7$ \\
$K_{Ste12}$ & $-4.5 \sim 0.0$ \\
$k_{m1}$ & $-2.3 \sim 0.0$ \\
$k_{a1}$ & $-11.0 \sim 4.0$ \\
$Digs_{T}$ & $-2.8 \sim 1.0$ \\
$k_{a2}$ & $-11.0 \sim 4.0$ \\
$k_{a3}$ & $-12.0 \sim 6.8$ \\
$k_{a5}$ & $-9.5 \sim -0.7$ \\
$K_{Far1}$ & $-4.5 \sim 0.0$ \\
$k_{p3}$ & $-5.0 \sim -1.0$ \\
$k_{p4}$ & $-11.0 \sim 0.0$ \\
$k_{degPF1}$ & $-3.7 \sim -0.5$ \\
$k_{synGFP}$ & $-7.1 \sim -0.6$ \\
$k_{a4}$ & $-10.5 \sim 0.0$ \\
$K_{GFP}$ & $-4.5 \sim 0.0$ \\
$k_{degGFP}$ & $-1.0$ (fixed) \\

\bottomrule
\end{tabular}
\end{sc}
\end{small}
\end{center}
\end{table}

\newpage

\section*{Acknowledgments}
J.K.K. is supported by the Institute for Basic Science IBS-R029-C3 and Samsung Science and Technology Foundation SSTF-BA1902-01. H.J. is supported by the National Research Foundation of Korea RS-2024-00357912 and Korea University grant K2411031. K.J. was supported by National Institutes of Health grant 1R01GM144959-01 and National Science Foundation grant DBI-1707400.

\bibliographystyle{unsrt}
\bibliography{main}
\end{document}

%% file: ex_shared.tex
\usepackage{lipsum}
\usepackage{amsfonts}
\usepackage{graphicx}
\usepackage{epstopdf}
\usepackage{algorithmic}
\DeclareMathOperator*{\argmin}{argmin}

\ifpdf
  \DeclareGraphicsExtensions{.eps,.pdf,.png,.jpg}
\else
  \DeclareGraphicsExtensions{.eps}
\fi

\newsiamremark{hypothesis}{Hypothesis}
\crefname{hypothesis}{Hypothesis}{Hypotheses}
\newsiamthm{claim}{Claim}

\headers{Parameter Estimation of Systems with Abrupt Signal}{H. Jo, K. Josi\'{c}, and J. K. Kim}

\title{Neural Network-Based Parameter Estimation for Non-Autonomous Differential Equations with Discontinuous Signals\thanks{Submitted to the editors, \today.}}

\author{
Hyeontae Jo\textsuperscript{1,}\textsuperscript{2}, 
Kre\v{s}imir Josi\'{c}\textsuperscript{3,}\textsuperscript{4,}\textsuperscript{$\dagger$}, 
Jae Kyoung Kim\textsuperscript{2,}\textsuperscript{5,}\textsuperscript{$\dagger$}
}

\date{}

\footnotetext[1]{Division of Applied Mathematical Sciences, Korea University, Sejong City, Republic of Korea}
\footnotetext[2]{Pioneer Research Center for Mathematical and Computational Sciences, Institute for Basic Science, Daejeon, Republic of Korea}
\footnotetext[3]{Departments of Mathematics, and Biology and Biochemistry, University of Houston, Houston, TX, USA}
\footnotetext[4]{Department of BioSciences, Rice University, Houston, TX, USA}
\footnotetext[5]{Department of Mathematical Sciences, KAIST, Daejeon, Republic of Korea}

\footnotetext[2]{Corresponding authors: K. Josi\'{c} (\texttt{josic@math.uh.edu}), J. K. Kim (\texttt{jaekkim@kaist.ac.kr})}

\usepackage{amsopn}

\makeatletter
\newcommand*{\addFileDependency}[1]{
  \typeout{(#1)}
  \@addtofilelist{#1}
  \IfFileExists{#1}{}{\typeout{No file #1.}}
}
\makeatother

\newcommand*{\myexternaldocument}[1]{%
    \externaldocument{#1}%
    \addFileDependency{#1.tex}%
    \addFileDependency{#1.aux}%
}